\documentclass{article} 
\usepackage{iclr2022_conference,times}

\usepackage{amsmath,amsfonts,bm}

\def\eqref#1{equation~\ref{#1}}

\def\1{\bm{1}}

\DeclareMathAlphabet{\mathsfit}{\encodingdefault}{\sfdefault}{m}{sl}
\SetMathAlphabet{\mathsfit}{bold}{\encodingdefault}{\sfdefault}{bx}{n}

\newcommand{\R}{\mathbb{R}}

\usepackage{hyperref}
\usepackage{url}

\usepackage{comment}
\usepackage{booktabs}
\usepackage{graphicx}
\usepackage{diagbox}
\usepackage{amsmath,amsfonts,graphicx,amssymb,bm,amsthm}
\usepackage{algorithm,algorithmicx}
\usepackage[noend]{algpseudocode}
\usepackage{fancyhdr}
\usepackage{tikz}
\usepackage{graphicx}
\usepackage{mathrsfs}
\usepackage{amsbsy}
\usepackage{bbm}
\usepackage{paralist}
\usepackage{subcaption}

\newtheorem{theorem}{Theorem}
\newtheorem{lemma}[theorem]{Lemma}

\newtheorem{corollary}[theorem]{Corollary}

\renewenvironment{proof}{{\noindent\bf Proof. \quad}}{\hfill$\qed$\par}

\title{Can Vision Transformers \\ Perform Convolution?}

\author{Shanda Li \\
School of EECS, Peking University\\
\texttt{lishanda@pku.edu.cn} \\
\And
Xiangning Chen \\
Department of Computer Science, UCLA \\
\texttt{xiangning@cs.ucla.edu} \\
\And
Di He \\
Microsoft Research \\
\texttt{dihe@microsoft.com}\\
\And
Cho-Jui Hsieh \\
Department of Computer Science, UCLA \\
\texttt{chohsieh@cs.ucla.edu}
}

\iclrfinalcopy 
\begin{document}

\maketitle
\vspace{-13px}
\begin{abstract}
Several recent studies have demonstrated that attention-based networks, such as Vision Transformer (ViT), can outperform Convolutional Neural Networks (CNNs) on several computer vision tasks without using convolutional layers. This naturally leads to the following questions: Can a self-attention layer of ViT express any convolution operation? 
In this work, we prove that a single ViT layer with image patches as the input can perform any convolution operation constructively, where the multi-head attention mechanism and the relative positional encoding play essential roles.
We further provide a lower bound on the number of heads for Vision Transformers to express CNNs. Corresponding with our analysis, experimental results show that the construction in our proof can help inject convolutional bias into Transformers and significantly improve the performance of ViT in low data regimes. 

\end{abstract}
\vspace{-10px}
\section{Introduction}

Recently, the Transformer~\citep{vaswani2017attention} architecture has achieved great success in vision after it dominates the language domain~\citep{devlin2019bert, liu2019roberta}.
Equipped with large-scale pre-training or several improved training strategies, the Vision Transformer (ViT) can outperform CNNs on a variety of challenging vision tasks~\citep{dosovitskiy2021an, touvron2021deit, liu2021swin, chen2021vision}.
As Transformer takes 1D sequences of tokens as input, the common manner to transform a 2D image into such a 1D sequence is introduced by ~\citet{dosovitskiy2021an}:
An image $\boldsymbol{X} \in\R ^{H\times W\times C}$ is reshaped into a sequence of flattened patches $\boldsymbol{\tilde X}\in \mathbb{R}^{N\times P^2 C}$, where $H, W, C$ are the image height, width, channel, $P$ is the patch resolution, and $N=HW/P^2$ is the sequence length.
By using specific positional encoding to encode spatial relationship between patches, a standard Transformer can therefore be used in the vision domain.

It has been observed that when there is sufficient training data, ViT can dramatically outperform convolution-based neural network models~\citep{dosovitskiy2021an} (e.g., 85.6\% vs 83.3\% ImageNet top-1 accuracy for ViT-L/16 and ResNet-152x2 when pre-trained on JFT-300M).
However, ViT still performs worse than CNN when trained on smaller-scale datasets such as CIFAR-100. 
Motivated by these observations, it becomes natural to compare the expressive power of Transformer and CNN. Intuitively, a Transformer layer is more powerful since the self-attention mechanism enables context-dependent weighting while a convolution can only capture local features. However, it is still unclear whether a Transformer layer is strictly more powerful than convolution. In other words: 

\begin{center}
    \emph{Can a self-attention layer of ViT (with image patches as input) express any convolution operation? }
\end{center}

A partial answer has been given by \cite{cordonnier2020on}. They showed that a self-attention layer with a sufficient number of heads can express convolution, but they only focused on the settings where the input to the attention layer is the representations of \emph{pixels}, which is impractical due to extremely long input sequence and huge memory cost. In Vision Transformer and most of its variants~\citep{touvron2021deit, dosovitskiy2021an, dascoli2021convit}, the input is the representations of non-overlapping image \emph{patches} instead of pixels. As a convolution operation can involve pixels across patch boundaries, whether a self-attention layer in ViT can express convolution is still unknown. 

In this work, we give an affirmative answer to the above-mentioned question. 
We formally prove that a ViT layer with relative positional encoding and sufficient attention heads can express any convolution even when the input is image patches. This implies that the poor performance of ViT on small datasets is mainly due to its generalization ability instead of expressive power. We further provide a lower bound on the number of heads required for transforming convolution into a self-attention layer. Based on our theoretical findings, we propose a two-phase training pipeline to inject convolutional bias into Vision Transformers, and empirically demonstrate its effectiveness in low data regimes.

The contributions of this paper are summarized below.

\begin{itemize}
    \item We provide a constructive proof to show that a 9-head self-attention layer in Vision Transformers with image patch as the input can perform any convolution operation, where the key insight is to leverage the multi-head attention mechanism and relative positional encoding to aggregate features for computing convolution. 
    
    \item We prove lower bounds on the number of heads for self-attention layers to express convolution operation, for both the patch input and the pixel input setting. This result shows that the construction in the above-mentioned constructive proof is optimal in terms of the number of heads. 
    Specifically, we show that 9 heads are both necessary and sufficient for a self-attention layer with patch input to express convolution with a $K\times K$ kernel, while a self-attention layer with pixel input must need $K^2$ heads to do so. Therefore, Vision Transformers with patch input are more \emph{head-efficient} than pixel input when expressing convolution.
    
    \item We propose a two-phase training pipeline for Vision Transformers. The key component in this pipeline is to initialize ViT from a well-trained CNN using the construction in our theoretical proof. We empirically show that with the proposed training pipeline that explicitly injects the convolutional bias, ViT can achieve much better performance compared with models trained with random initialization in low data regimes. 
\end{itemize}

\section{Preliminaries}

In this section, we recap the preliminaries of Convolutional Neural Networks and Vision Transformers, and define the notations used in our theoretical analysis. We use bold upper-case letters to denote matrices and tensors, and bold lower-case letters to denote vectors. Let $[m]=\{1,2, \cdots, m\}$. 
The indicator function of $A$ is denoted by $\mathbbm{1}_A$. 

\subsection{Convolutional Neural Networks}
Convolutional Neural Networks (CNNs) are widely used in computer vision tasks, in which the convolutional layer is the key component.

\paragraph{Convolutional layer.} Given an image $\boldsymbol{X} \in\R ^{H\times W\times C}$, the output of a convolutional layer for pixel $(i,j)$ is given by
\begin{equation}\label{eqn:def_conv}
    \mathrm{Conv}(\boldsymbol X)_{i,j,:} = \sum_{(\delta_1, \delta_2)\in \Delta} \boldsymbol X_{i+\delta_1,j+\delta_2,:} \boldsymbol W_{\delta_1,\delta_2,:, :}^C,
\end{equation}
where $\boldsymbol W^C\in\R^{K\times K\times C\times D_{out}}$ is the learnable convolutional kernel, $K$ is the size of the kernel and the set $\Delta=\{-\lfloor K/2 \rfloor, \cdots, \lfloor K/2 \rfloor\}\times \{-\lfloor K/2 \rfloor, \cdots, \lfloor K/2 \rfloor\}$ is the receptive field.

\subsection{Vision Transformers}

A Vision Transformer takes sequences of image patches as input. It usually begins with a patch projection layer, followed by a stack of Transformer layers. A Transformer layer contains two sub-layers: the multi-head self-attention (MHSA) sub-layer and the feed-forward network (FFN) sub-layer. Residual connection~\citep{he2016deep} and layer normalization~\citep{lei2016layer} are applied for both sub-layers individually. Some important components are detailed as follows: 

\paragraph{Patch input.} Consider an input image $\boldsymbol{X} \in\R ^{H\times W\times C}$, where $H, W, C$ is the image height, width and channel. To feed it into a Vision Transformer, it is reshaped into a sequence of flattened patches $\boldsymbol{\tilde X} \in\R ^{N\times P^2C}$, where $P$ is the patch resolution, and $N=HW/P^2$ is the sequence length. Formally, a flattened patch is defined as $
    \boldsymbol{\tilde X}_{i,:} = \mathrm{concat}\left(
        \boldsymbol X_{h_{i1},w_{i1},:}, \cdots, \boldsymbol X_{h_{iP^2},w_{iP^2},:}
    \right)$,
where $(h_{i1},w_{i1}), \cdots, (h_{iP^2},w_{iP^2})$ are the positions of pixels in the $i$-th patch. Then a linear projection is applied on all flattened patches to obtain the input to the Transformer.

\paragraph{Multi-head self-attention (MHSA) layer.} 
The attention module is formulated as querying a dictionary with key-value pairs, i.e., 
$ \mathrm{Attention}(\boldsymbol Q,\boldsymbol K, \boldsymbol V) = \mathrm{softmax}\left( \frac{\boldsymbol {QK}^{\top}}{\sqrt d} \right) \boldsymbol V,$
where $d$ is the dimension of the hidden representations, and $\boldsymbol Q$, $\boldsymbol K$, $\boldsymbol V$ are referred to as queries, keys and values that are all produced by linearly projecting the output of the previous layer. The multi-head variant of the attention module is popularly used because it allows the model to jointly learn the information from different representation sub-spaces. 
Formally, an MHSA layer with input $\boldsymbol{H}\in\R^{N\times d}$ is defined as:
\begin{align}
    \mathrm{MHSA}(\boldsymbol H)&=\mathrm{concat}(\mathrm{SA}_1(\boldsymbol H),\cdots, \mathrm{SA}_{N_H}(\boldsymbol H))\boldsymbol W^O=\sum_{k=1}^{N_H} \mathrm{SA}_k(\boldsymbol H)\boldsymbol W^O_k \label{eqn:def_mhsa}\\
    \mathrm{SA}_k(\boldsymbol H)& = \mathrm{Attention}(\boldsymbol{HW}^ Q_k,\boldsymbol{HW}^K_k, \boldsymbol{HW}^V_k),
\end{align}
where $\boldsymbol{W}^ Q_k,\boldsymbol{W}^K_k, \boldsymbol{W}^V_k\in \R ^{d\times d_H}$ and $\boldsymbol{W}^O=(\boldsymbol{W}^{O\top}_1, \cdots,\boldsymbol{W}^{O\top}_{N_H})^{\top} \in \R ^{N_Hd_H\times d_O}$ are learnable projection matrices\footnote{For simplicity, the bias terms of linear projections are omitted.}, $N_H$ is the number of heads, $d_H$ is the size of each head, and $d_O$ is the dimensionality of the output.


\paragraph{Relative positional encoding.} 
Many Vision Transformer models adopt a learnable relative position bias term in computing self-attention scores~\citep{liu2021swin, luo2021stable,li2021learnable}:
\begin{equation}\label{eqn:rpe_attn}
    \mathrm{Attention}(\boldsymbol Q,\boldsymbol K, \boldsymbol V) = \mathrm{softmax}\left( \frac{\boldsymbol {QK}^\top}{\sqrt d} + \boldsymbol B\right) \boldsymbol V,
\end{equation}
where $\boldsymbol B_{i,j}$ only depends on the \emph{relative position} between the $i$-th patch (query patch) and the $j$-th patch (key patch). More specifically, assume the position of the $\ell$-th patch is $(x_{\ell}, y_{\ell})$, then $\boldsymbol B_{i,j}=b_{(x_i-x_j, y_i-y_j)}$. For $-\frac{H}{P}+1\leq x\leq \frac{H}{P}-1$ and $-\frac{W}{P}+1\leq y\leq \frac{W}{P}-1$, $b_{(x,y)}$ is a trainable scalar.



\section{Expressing convolution with the MHSA layer}

In this section, we consider the question of using the MHSA layer in Vision Transformers to express a convolutional layer. We mainly focus on the patch-input setting, which is more realistic for current ViTs. 
First, we show that a MHSA layer in Vision Transformers can express a convolutional layer in the patch-input setting (Theorem \ref{thm:patch_theorem}). 
Second, we prove lower bounds on the number of heads for self-attention layers to express the convolution operation for both patch and pixel input settings, which demonstrates that the number of heads required in Theorem \ref{thm:patch_theorem} is optimal. Putting the representation theorem and the lower bounds together, we conclude that MHSA layers with patch input are more \emph{head-efficient} in expressing convolutions. The dependency on the number of heads is more feasible in the patch-input setting for Vision Transformers in practice. 

\subsection{An MHSA layer with enough heads can express a convolutional layer}

Our main result in this subsection is that an MHSA layer can express convolution under mild assumptions in the patch-input setting. To be precise, we present the following theorem:

\begin{theorem}\label{thm:patch_theorem}
    In the patch-input setting, assume $d_H\geq d$ and $d_O\geq P^2D_{out}$. Then a multi-head self-attention layer with $N_H=\left(2 \left\lceil \frac{K-1}{2P} \right\rceil +1\right)^2$ heads and relative positional encoding can express any convolutional layer of kernel size $K \times K$, and $D_{out}$ output channels.
\end{theorem}

The patch input poses the major difficulty in proving this result: The convolution operation can involve pixels across patch boundaries, which makes the problem complicated.
To address this, we first aggregate the information from all the relevant patches for calculating the convolution by leverage the relative positional encoding and multi-head mechanism, and then apply a linear projection on the aggregated features. This idea leads to a constructive proof.

\paragraph{Proof sketch of Theorem \ref{thm:patch_theorem}.} 
Note that the attention calculation with relative positional encoding can be dissected into a context-aware part (which depends on all the input tokens) and a positional attention part (which is agnostic to the input): 
In Equation (\ref{eqn:rpe_attn}), $\boldsymbol {QK}^\top$ and $\boldsymbol B$ correspond to the context-aware part and the positional attention part respectively. Since convolution is context-agnostic by nature, we set $\boldsymbol {W}^Q_k=\boldsymbol {W}^K_k=\boldsymbol{0}~(\forall~k\in[N_H])$ and purely rely on the positional attention in the proof. 
Given any relative position $\delta$ between two patches, we force the query patch to focus on \emph{exactly one} key patch, such that the relative position between the query and the key is $\delta$. We elaborate on this argument in Lemma \ref{lemma:rpe_attn}.

\begin{lemma}\label{lemma:rpe_attn}
    For any relative position $\delta$ between two patches, there exists a relative positional encoding scheme $\boldsymbol B$ such that $\mathrm{softmax}(\boldsymbol B_{q,:})_k=\mathbbm{1}_{\{q-k=\delta\}}$, where $q,k$ are the index of the query/key patch.\footnote{Here we abuse the notation for ease of illustration: When used as subscripts, $q,k$ are scalars in $[N]$; When used to denote the locations of patches, $q,k$ are two-dimensional coordinates in $[H/P]\times [W/P]$.} 
\end{lemma}

Let the receptive field of a given patch in $K\times K$ convolution be the set of patches that contain at least one pixel in the receptive field of any pixel in the given patch. Then it's easy to see that the relative position between a given patch and the patches in its receptive field are
\begin{equation}
    \tilde \Delta = \left\{-\left\lceil \frac{K-1}{2P} \right\rceil, \cdots, \left\lceil \frac{K-1}{2P} \right\rceil \right\} \times \left\{-\left\lceil \frac{K-1}{2P} \right\rceil, \cdots, \left\lceil \frac{K-1}{2P} \right\rceil\right\}.
\end{equation}

With Lemma \ref{lemma:rpe_attn}, we can force the query patch to attend to the patch at a given relative position in $\tilde\Delta$ in each head. By setting $\boldsymbol{W}^V_k=(\boldsymbol{I}_{d}, \boldsymbol{0}_{d\times(d_H-d)})$, the hidden representation (before the final projection $\boldsymbol{W}^O$) of the query patch contains the features of all the patches in its receptive field. 

Finally, by the linearity of convolution, we can properly set the weights in $\boldsymbol{W}^O$ based on the convolution kernel, such that the final output is equivalent to that of the convolution for any pixel, which concludes the proof. We refer the readers interested in a formal proof to Appendix \ref{appendix:proof_the_thm}.\hfill$\qed$

\paragraph{Remark on the positional encoding.} In this result, we focus on a specific form of relative positional encoding. 
In fact, Theorem \ref{thm:patch_theorem} holds as long as the positional encoding satisfies the property in Lemma \ref{lemma:rpe_attn}. It's easy to check that a wide range of positional encoding have such property \citep{dai2019transformer,raffel2020exploring, ke2020rethinking,liu2021swin}, so our result is general. 

However, our construction does not apply to MHSA layers that only use \emph{absolute positional encoding}. In the prood, we need a \emph{separate context-agnostic term} in calculating attention scores. However, absolute positional encoding, which is typically added to the input representation, cannot be separated from context-dependent information and generate the desired attention pattern in Lemma~\ref{lemma:rpe_attn}.

\paragraph{Remark on the pixel-input setting.} It should be noted that the pixel-input setting is a special case of the analyzed patch-input setting, since patches become pixels when patch resolution $P=1$. Therefore, the result in \citet{cordonnier2020on} can be viewed as a natural corollary of Theorem \ref{thm:patch_theorem}.

\begin{corollary}\label{cor:pixel_theorem}
    In the pixel-input setting, a multi-head self-attention layer with $N_H=K^2$ heads of dimension $d_H$, output dimension $d_O$ and relative positional encodings can express any convolutional layer of kernel size $K \times K$ and $\min\{d_H, d_O\}$ output channels.
\end{corollary}

\paragraph{Practical implications of Theorem \ref{thm:patch_theorem}.} For Vision Transformers and CNNs used in practice, we typically have $K<2P$, e.g., $P\geq 16$ in most Vision Transformers, and $K=3,5,7$ in most CNNs. Thus, the following corollary is more practical:

\begin{corollary}\label{cor:9heads}
    In the patch-input setting, assume $K<2P$, $d_H\geq d$ and $d_O\geq P^2D_{out}$. Then a multi-head self-attention layer with \textbf{9 heads} and relative positional encoding can express any convolutional layer of kernel size $K \times K$ and $D_{out}$ output channels.
\end{corollary}

Another thing that would be important from a practical perspective is that Theorem \ref{thm:patch_theorem} can be generalized to other forms of convolution operations, although we focus on the simplest formulation defined in Equation \ref{eqn:def_conv}. For example, people sometimes use convolution with stride greater than 1, or dilated convolution \citep{yu2015multi} in practice. This theorem can be easily generalized to these cases. Intuitively, we only use the linearity of convolution in our proof, so any variant of convolution that preserves this property can be expressed by MHSA layers according to our construction.


\subsection{An MHSA layer with insufficient heads cannot express convolution}

It's noticeable that the multi-head mechanism plays an essential role in the constructive proof of Theorem \ref{thm:patch_theorem}. Thus, it's natural to ask whether the dependency on the number of heads is optimal in the theorem. In this part, we present lower bounds on the number of heads required for an MHSA layer to express convolution in both pixel-input and patch-input setting, highlighting the importance of the multi-head mechanism and showing the optimality of our construction in the previous proof.

\subsubsection{The pixel-input setting}
\label{sec:pixel_lower_bound}
We fisrt show that the dependency on the number of heads in Corollary \ref{cor:pixel_theorem} is optimal, i.e., an MHSA layer must need $K^2$ heads to express convolution of kernel size $K \times K$.

\begin{theorem}\label{thm:pixel_lower_bound}
    In the pixel-input setting, suppose $N_H<\min\{K^2, d\}$. There exists a convolutional kernel weight $\boldsymbol{W^C}\in \R^{K\times K\times d\times D_{out}}$ such that any MHSA layer with $N_H$ heads and relative positional encoding cannot express $\mathrm{conv}(\cdot; \boldsymbol{W^C})$.
\end{theorem}

\begin{proof}
    We will prove the theorem in the case where $D_{out}=1$ by contradiction, and consequently the result will hold for any $D_{out}\in \mathbb{N}^*$. Since $D_{out}=1$, we view $\boldsymbol{W^C}$ as a three-dimensional tensor.

    In the convolutional layer, consider the output representation of the pixel at position $\gamma\in[H]\times[W]$:
    \begin{equation}\label{eqn:proof_conv}
        \mathrm{conv}(\boldsymbol{X}; \boldsymbol{W^C})_{\gamma}=\sum_{\delta \in \Delta}\sum_{i=1}^d \boldsymbol{X}_{\gamma+\delta,i} \boldsymbol{W}^{C}_{\delta,i},
    \end{equation}
    where $\Delta=\{-\lfloor K/2 \rfloor, \cdots, \lfloor K/2 \rfloor\}\times \{-\lfloor K/2 \rfloor, \cdots, \lfloor K/2 \rfloor\}$.
    
    In the MHSA layer, assume the attention score between the query pixel $\gamma$ and the key pixel $\gamma+\delta$ in the $k$-th head is $a_{\delta}^k(\gamma)$. 
    Let $\boldsymbol{W}^V_k\boldsymbol{W}^O_k=\boldsymbol{w}^k=(w_{1}^k, \cdots, w_{d}^k)^{\top}\in \R^{d\times D_{out}}$ (recall that $C_{out }=1$). 
    Then, the output representation of the pixel at position $\gamma\in[H]\times[W]$ is
    \begin{equation}
        \mathrm{MHSA}(\boldsymbol{X})_{\gamma}=\sum_{k=1}^{N_H}\sum_{\delta} a_{\delta}^h(\gamma) \sum_{i=1}^d \boldsymbol{X}_{\gamma+\delta,i} w_{i}^k
        =\sum_{\delta}\sum_{i=1}^d \boldsymbol{X}_{\gamma+\delta,i} \sum_{k=1}^{N_H} a_{\delta}^k(\gamma)  w_{i}^k.
        \label{eqn:proof_mhsa}
    \end{equation}
    
    Putting Equation \ref{eqn:proof_conv} and \ref{eqn:proof_mhsa} together, in order to ensure $\mathrm{conv}(\boldsymbol{X}, \boldsymbol{W^C})_{\gamma}=\mathrm{MHSA}(\boldsymbol{X})_{\gamma}$, we have
    \begin{equation}
        \boldsymbol W^{C}_{\delta,i} = \sum_{k=1}^{N_H} a_{\delta}^k(\gamma) w_{i}^k ~ (\forall \delta\in\Delta, i\in[d])
        \Rightarrow   \boldsymbol{\tilde W}^C = \sum_{k=1}^{N_H} \boldsymbol{a}^{k}(\gamma) \boldsymbol{w}^{k\top},
    \end{equation} 
    where $\boldsymbol{a}^{k}(\gamma)=(a_{\delta}^k(\gamma))_{\delta \in \Delta} \in\R^{K^2}$ is a row vector for any $k\in[N_H]$, and $\boldsymbol{\tilde W}^C\in\R^{K^2\times d}$ is reshaped from the weights $\boldsymbol{W}^C\in \R^{K\times K\times d\times 1}$.

    Note that
    \begin{equation}
        rank\left(\sum_{k=1}^{N_H} \boldsymbol{a}^{k}(\gamma) \boldsymbol{w}^{k\top} \right) \leq N_H < \min\{K^2,d\}.
    \end{equation}

    By properly choosing convolutional kernel weights $\boldsymbol{W}^C$ such that $rank(\boldsymbol{\tilde W}^C)=\min\{K^2,d\}$, we conclude the proof by contradiction.
\end{proof}

\paragraph{Remark.} For Vision Transformers and CNNs used in practice, we typically have $\min\{K^2,d\}=K^2$. Thus this result shows that $K^2$ heads are necessary, and Corollary \ref{cor:pixel_theorem} is optimal in terms of the number of heads.

\subsubsection{The patch-input setting}

In the patch-input setting, we show that at least 9 heads are needed for MHSA layers to perform convolution.

\begin{theorem}\label{thm:patch_lower_bound}
    In the patch-input setting, suppose $K \geq 3$ and $N_H\leq 8$. There exists a convolutional kernel weight $\boldsymbol{W}^C\in \R^{K\times K\times D_{in}\times D_{out}}$, such that any MHSA layer with $N_H$ heads and relative positional encoding cannot express $\mathrm{conv}(\cdot; \boldsymbol{W}^C)$.
\end{theorem}

Similar to Theorem \ref{thm:pixel_lower_bound}, this theorem is also proven with a rank-based argument.
However, the proof requires more complicated techniques to deal with the patch input, so we defer it to Appendix \ref{appendix:proof_patch_lower_bound}.

\paragraph{Remark.} This result shows that Corollary \ref{cor:9heads} is also optimal in terms of the number of heads in practical cases.

\paragraph{Discussions on the theoretical findings.} Our findings clearly demonstrate the difference between the pixel-input and patch-input setting: \emph{patch input makes self-attention require less heads to perform convolution compared to pixel input, especially when $K$ is large}. For example, according to Theorem \ref{thm:pixel_lower_bound}, MHSA layers with pixel input need at least $25$ heads to perform $5\times 5$ convolution, while those with patch input only need $9$ heads. Usually the number of heads in a MHSA layer is small in Vision Transformers, e.g., there are only $12$ heads in ViT-base. Therefore, our theory is realistic and aligns well with practical settings.

\section{Two-phase training of Vision Transformers}

Our theoretical results provide a construction that allows MHSA layers to express convolution. In this section, we propose a two-phase training pipeline for Vision Transformers which takes advantage of the construction. Then we conduct experiments using this pipeline and demonstrate that our theoretical insight can be used to inject convolutional bias to Vision Transformers and improve their performance in low data regimes. We also discuss additional benefits of the proposed training pipeline from the optimization perspective. Finally, we conclude this section with a discussion on the limitation of our method.

\subsection{Method and implementation details}
\label{sec:exp_method}
\paragraph{Two-phase training pipeline.} Inspired by the theoretical findings, we propose a two-phase training pipeline for Vision Transformers in the low data regime, which is illustrated in Figure \ref{fig:two_phase}. Specifically, we first train a ``convolutional'' variant of Vision Transformers, where the MHSA layer is replaced by a $K\times K$ convolutional layer. We refer to this as the \textbf{convolution phase} of training. After that, we transfer the weights in the pre-trained model to a Transformer model, and continue training the model on the same dataset. We refer to this as the \textbf{self-attention phase} of training. The non-trivial step in the pipeline is to initialize MHSA layers from well-trained convolutional layers, and we utilize the construction in the proof of Theorem \ref{thm:patch_theorem} to do so. Due to the existence of the convolution phase, we cannot use a \texttt{[cls]} token for classification. Instead, we follow \citet{liu2021swin} to perform image classification by applying global average pooling over the output of the last layer, followed by a linear classifier. This method is commonly used in CNNs for image classification.

Intuitively, in the convolution phase, the model learns a ``convolutional neural network'' on the data and enjoys the inductive bias including locality and spatial invariance which makes learning easier. In the self-attention phase, the model mimics the pre-trained CNN in the beginning, and gradually learns to leverage the flexibility and strong expressive power of self-attention.

\paragraph{Implementation details.} While our theory focuses on a single MHSA layer, we experiment with 6-layer Vision Transformers to show that our theoretical insight still applies when there are stacked Transformer layers. 
We focus on the low-data regime and train our model on CIFAR-100 \citep{krizhevsky2009learning}. The input resolution is set to 224, and the patch resolution $P$ is set to 16.

In the convolution phase, we experiment models with convolutional kernel size $K=3\ \text{and}\ 5$. To apply our theory, in the self-attention phase, the number of attention heads $N_H$ is set to 9. The input and output dimension of MHSA layers $d$ and $d_O$ are both set to 768. The size of each head $d_H$ is set to 768. The dimension of feed-forward layer $d_{FFN}$ is set to 3072. 
Detailed descriptions of the experimental settings are presented in Appendix \ref{appendix:exp}.


\begin{figure}[t]
\begin{minipage}{0.61\linewidth}
    \centering
    \includegraphics[width=0.87\linewidth]{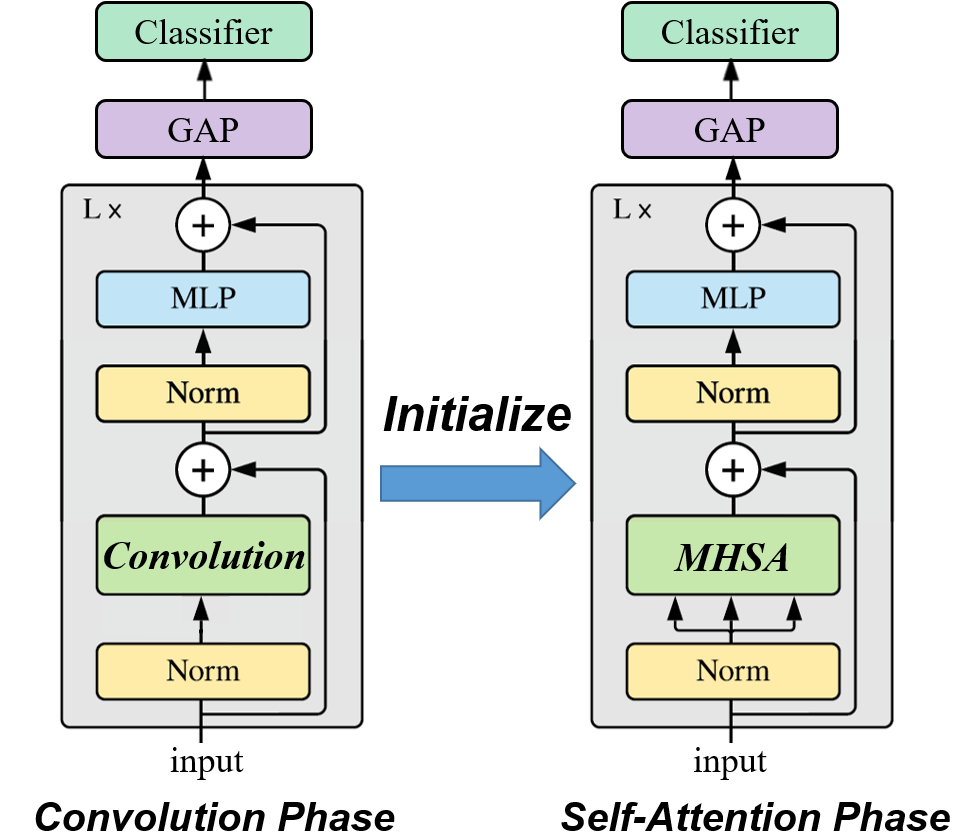}
    \caption{Overview of the two-phase training pipeline. See Section \ref{sec:exp_method} for details.}
    \label{fig:two_phase}
\end{minipage}\hfill
\begin{minipage}{0.34\linewidth}
    \centering
    \begin{subfigure}[t]{1\textwidth}
        \centering
        \includegraphics[width=1\linewidth]{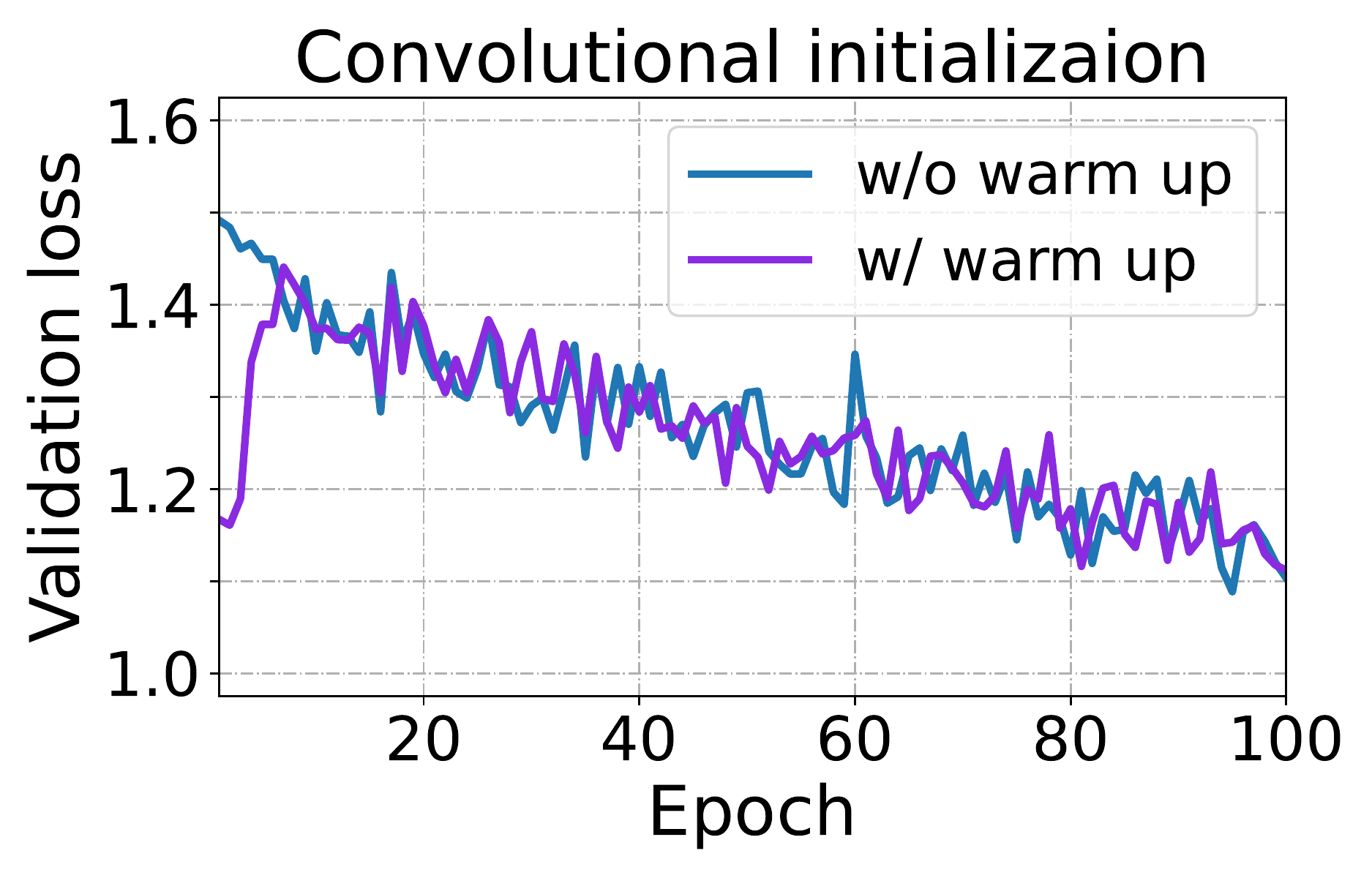}
    \end{subfigure}
    \begin{subfigure}[t]{1\textwidth}
        \centering
        \includegraphics[width=1\linewidth]{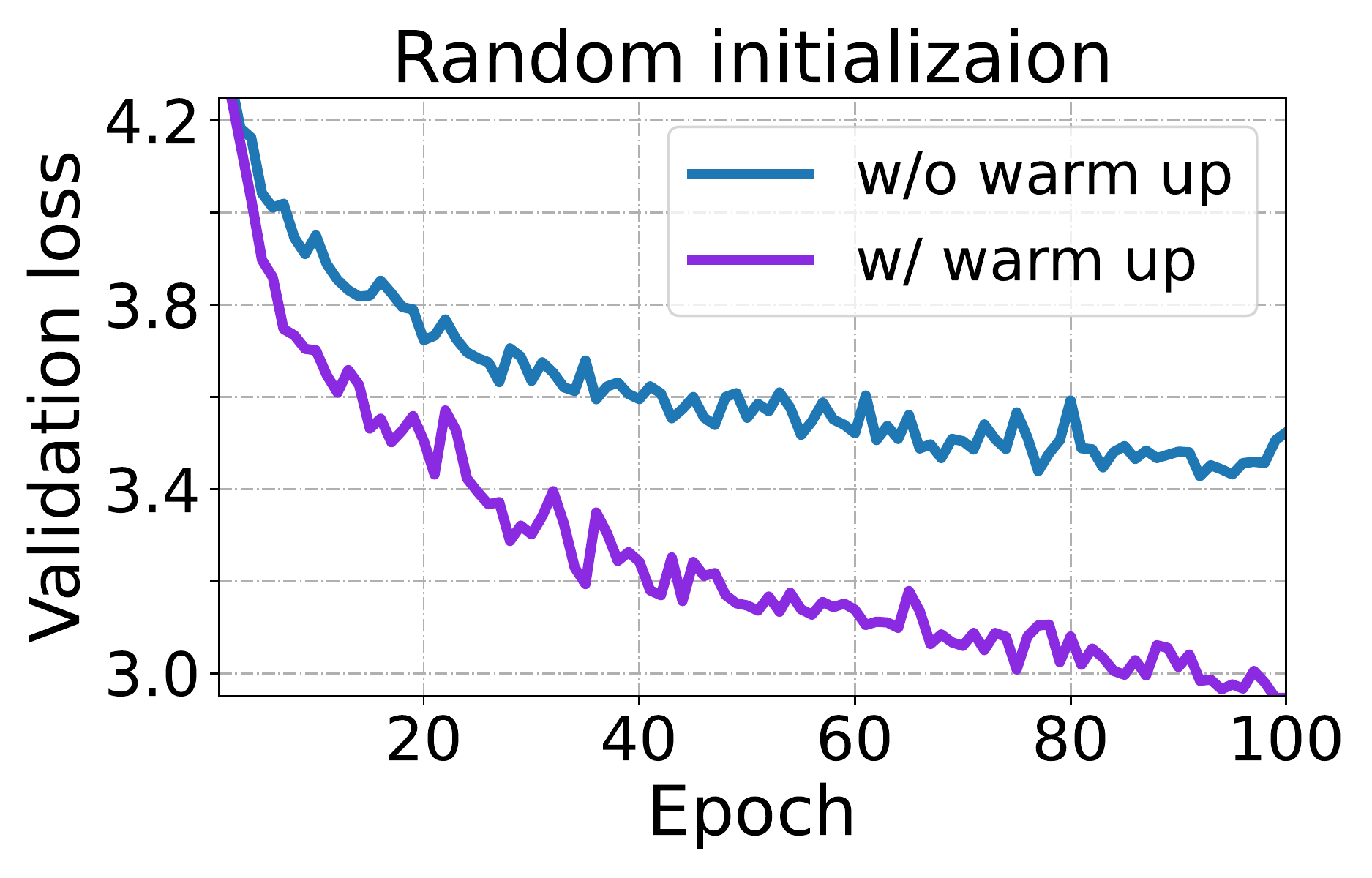}
    \end{subfigure}
    \caption{Loss curves of CMHSA with or without the warm-up stage under two initialization schemes.}
    \label{fig:warm_up}
\end{minipage}
\end{figure}

\subsection{Experimental results}

For ease of illustration, we name our models as CMHSA-$K$ (Convolutionalized MHSA) where $K$ is the size of the convolutional kernel in the first training phase. 
To demonstrate the effectiveness of our approach, we choose several baselines models for comparison:
\begin{itemize}
    \item ViT-base proposed in \citep{dosovitskiy2021an}, which applies Transformers on image classification straightforwardly.
    
    \item DeiT-small and DeiT-base proposed in \citet{touvron2021deit}, which largely improve the performance of ViT using strong data augmentation and sophisticated regularization. 
    
    \item CMHSA trained only in convolution phase or self-attention phase. When training directly in self-attention phase, the model is initialized randomly. In Table \ref{tab:main_tab}, CMHSA-$K$ (1st phase) refers to models trained only in convolution phase, and CMHSA (2nd phase) refers to models trained only in self-attention phase (Note that $K$ is irrelevant in this case).
\end{itemize}

To ensure a fair comparison, all the baseline models are trained for 400 epochs, while our models are trained for 200 epochs in each phase.

The experimental results are shown in Table \ref{tab:main_tab}. We evaluate the performance of the models in terms of both test accuracy and training cost, and make the following observations on the results:

\paragraph{The proposed two-phase training pipeline largely improves performance.} It's easy to see that DeiTs clearly outperform ViT, demonstrating the effectiveness of the training strategy employed by DeiT. Furthermore, our models with two-phase training pipeline outperform DeiTs by a large margin, e.g., the top-1 accuracy of our CMHSA-5 model is nearly \emph{9\% higher} than that of DeiT-base. This demonstrates that the proposed training pipeline can provide further performance gain on top of the data augmentation and regularization techniques in the low-data regime.

\paragraph{Both training phases are important.} From the last 5 rows of Table \ref{tab:main_tab}, we can see that under the same number of epochs, CMHSAs trained with only one phase always underperform those trained with both two phases. The test accuracy of CMHSA (2nd phase), which is a randomly initialized CMHSA trained for 400 epochs, is much lower than that of our final models (which are trained in both two phases). Therefore, the convolutional bias transferred from the first phase is crucial for the model to achieve good performance. Besides, the models trained only in the first phase are also worse than our final models. For example, CMHSA-5 outperforms CMHSA-5 (1st phase) by 2.62\%. This shows that the second phase enables the model to utilize the flexibility of MHSA layers to learn better representations, achieving further improvements upon the convolutional inductive bias.

\paragraph{The convolutional phase helps to accelerate training.} Training of Transformers is usually time-consuming due to the high computational complexity of the MHSA module. In contrast, CNNs enjoy much faster training and inference speed. In our proposed training pipeline, the convolution phase is very efficient. Although the self-attention phase is slightly slower, we can still finish 400 epochs of training using less time compared with DeiT-base. In Table \ref{tab:main_tab}, it is clear that our model significantly outperforms other models with a comparable training time.

\begin{table}[t]
\vspace{-10pt}
\caption{Experimental results on CIFAR-100 dataset. }\label{tab:main_tab}
\centering
\begin{tabular}{cccccccc}\toprule
Model         & $L$ & $N_H$ & $d$  & $d_{FFN}$ & Top-1 & Top-5 & Training time \\ \midrule
ViT-base      & 12  & 12    & 768  & 3072      & 60.90 & 86.66 & 1.00$\times$  \\
DeiT-small    & 12  & 6     & 384  & 1536      & 71.83 & 90.99 & 0.57$\times$  \\
DeiT-base     & 12  & 12    & 768  & 3072      & 69.98 & 88.91 & 1.00$\times$  \\ \midrule
CMHSA-3 (1st phase)& 6   & $-$   & 768  & 3072      & 76.07 & 93.03 & 0.45$\times$  \\
CMHSA-5 (1st phase)& 6   & $-$   & 768  & 3072      & 76.12 & 93.13 & 0.49$\times$  \\ \midrule
CMHSA (2nd phase)& 6   & 9     & 768  & 3072      & 69.83 & 91.39 & 1.48$\times$  \\ \midrule
CMHSA-3 (ours)& 6   & 9     & 768  & 3072      & 76.72 & 93.74 & 0.96$\times$  \\
CMHSA-5 (ours)& 6   & 9     & 768  & 3072      & \textbf{78.74} & \textbf{94.40} & 0.98$\times$   \\
\bottomrule     
\end{tabular}
\end{table}

\subsection{Additional benefits of the two-phase training pipeline}

As mentioned above, the two-phase training pipeline helps improve training efficiency and test accuracy of Vision Transformers. In this subsection we emphasize an additional benefit of the method: \textbf{The injected convolutional bias makes the optimization process easier}, allowing us to remove the warm-up epochs in training.

The warm-up stage is crucial to stabilize the training of Transformers and improve the final performance, but it also slows down the optimization and brings more hyperparameter tuning \citep{huang2020improving, xiong2020layer}. 
We empirically show that, when initialized with a pretrained CNN, our CMHSA model can be trained \emph{without} warm-up and still obtains competitive performance.

We train CMHSA in the self-attention phase with pretrained convolution-phase initialization and random initialization, and show the loss curves of the first 100 epochs in Figure \ref{fig:warm_up}. In the figure, the $x$-axis indicates the number of epoch and the $y$-axis indicates the validation loss.
``Convolutional Initialization'' refers to the models initialized from a pretrained convolution-phase, while ``Random Initialization'' refers to the models initialized randomly.
The only difference in experimental setting between the ``w/o warm up'' and ``w/ warm up'' curves is whether a warm-up stage is applied, and all the other hyperparameters are unchanged.

From Figure \ref{fig:warm_up}, we can see that when CMHSA is initialized randomly in the self-attention phase, the training is ineffective without the warm-up stage. For example, it takes nearly 80 epochs for the model without warm-up stage to reach the same validation loss achieved in the 20th epoch of the model with warm-up. 
In contrast, when initialized from the pretrained convolution phase, the validation losses are similar for models with and without the warm-up stage. After 200 epochs' training, the model without warm-up stage achieves $78.90\%$ top-1 accuracy, slightly outperforming the model with warm-up stage ($78.74\%$). Therefore, the proposed two-phase training pipeline can take advantage of the convolutional inductive bias and make training of Vision Transformers easier, while enables to remove the warm-up stage and eases the efforts of hyperparameter tuning.

\paragraph{Limitations.} 
Finally, we point out that our current method cannot enable any ViTs to mimic CNNs since we have some constraints on the ViT architecture. In particular, we require sufficient number of heads ($\geq 9$). As suggested by our theory, an exact mapping doesn't exist for smaller number of heads, and it would be interesting to study how to properly initialize ViT from CNN even when the exact mapping is not applicable. 

\section{Related Work and Discussions}

\subsection{Expressive power of self-attention layers and Transformers}

As Transformers become increasingly popular, many theoretical results have emerged to study their expressive power. 
Transformers are Turing complete (under certain assumptions)~\citep{perez2018on} and universal approximators~\citep{yun2019transformers}, and these results have been extended to Transformer variants with sparse attention~\citep{zaheer2020big, yun2020on}. \cite{levine2020depth, wies2021which} study how the depth, width and embedding dimension affect the expressive power of Transformers. \cite{dong2021attention} analyze the limitations of a pure self-attention Transformer, illustrating the importance of FFN layers and skip connections.

There are also some works focusing on a single self-attention layer or the self-attention matrices. \cite{bhojanapalli2020low} identify a low-rank bottleneck in attention heads when the size of each head is small. \cite{likhosherstov2021expressive} proves that a fixed self-attention module can approximate arbitrary sparse patterns depending on the input when the size of each head $d=O(\log N)$, where $N$ denotes the sequence length. 

Our work is motivated by the recent success of Vision Transformers, and aims to compare a layer in Transformers and in CNNs, which is different from the works mentioned above. 
The most relevant work is \citep{cordonnier2020on}, which shows that a MHSA layer with $K^2$ heads can express a convolution of kernel size $K$. However, this result only focuses on the pixel-input setting, which is infeasible for current Vision Transformers to apply, especially on high-resolution images.
We study the more realistic patch-input setting (of which the pixel-input setting is a special case). We also derive lower bounds on the number of heads for such expressiveness, showing the optimality of our result. Therefore, our work provides a more precise picture showing how a MHSA layer in current ViTs can express convolution.

\subsection{Training Vision Transformers}
Despite Vision Transformers reveal extraordinary performance when pre-trained on large-scale datasets (e.g., ImageNet-21k and JFT-300M),
they usually lay behind CNNs when trained from scratch on ImageNet, let alone smaller datasets like CIFAR-10/100.
Previous methods usually employ strong augmentations~\citep{touvron2021deit} or sophisticated optimizer~\citep{chen2021vision} as rescues. 
For instance, ~\citet{chen2021vision} observe that enforcing the sharpness constraint during training can dramatically enhance the performance of Vision Transformers.
~\citet{touvron2021deit} stack multiple data augmentation strategies to manually inject inductive biases.
They also propose to enhance the accuracy by distilling Vision Transformers from pre-trained CNN teachers.

Supported by our theoretical analysis, we propose a two-phase training pipeline to inject convolutional bias into ViTs. The most relevant work to our approach is \citep{dascoli2021convit}. \citet{dascoli2021convit} propose a variant of ViT called ConViT, and they also try to inject convolutional bias into the model by initializing it following the construction in \citep{cordonnier2020on} so that the model can perform convolution operation at initialization, which resembles the second phase of our training pipeline.
However, their work differs from ours in several aspects: 
First, their initialization strategy only applies in the pixel-input setting. Thus the models can only perform convolution on images which are $16\times$ downsampled. By contrast, our construction enables MHSA layers with patch input to perform convolution on the \emph{original} image.
Second, they only initialize the attention module to express a \emph{random} convolution, while our method explicitly transfers information from a well-learned CNN into a ViT.
Third, ConViT makes architectural changes by introducing Gated Positional Self-Attention layers, while we keep the MHSA module unmodified. 
%

\section{Conclusion}
In this work, we prove that a single ViT layer can perform any convolution operation constructively, and we further provide a lower bound on the number of heads for Vision Transformers to express CNNs. Corresponding with our analysis, we propose a two phase training pipeline to help inject convolutional bias into Transformers, which improves test accuracy, training efficiency and optimization stability of ViTs in the low data regimes. 




\bibliography{iclr2022_conference}

\begin{thebibliography}{32}
\providecommand{\natexlab}[1]{#1}
\providecommand{\url}[1]{\texttt{#1}}
\expandafter\ifx\csname urlstyle\endcsname\relax
  \providecommand{\doi}[1]{doi: #1}\else
  \providecommand{\doi}{doi: \begingroup \urlstyle{rm}\Url}\fi

\bibitem[Bhojanapalli et~al.(2020)Bhojanapalli, Yun, Rawat, Reddi, and
  Kumar]{bhojanapalli2020low}
Srinadh Bhojanapalli, Chulhee Yun, Ankit~Singh Rawat, Sashank Reddi, and Sanjiv
  Kumar.
\newblock Low-rank bottleneck in multi-head attention models.
\newblock In \emph{International Conference on Machine Learning}, pp.\
  864--873. PMLR, 2020.

\bibitem[Chen et~al.(2021)Chen, Hsieh, and Gong]{chen2021vision}
Xiangning Chen, Cho-Jui Hsieh, and Boqing Gong.
\newblock When vision transformers outperform resnets without pretraining or
  strong data augmentations, 2021.

\bibitem[Cordonnier et~al.(2020)Cordonnier, Loukas, and
  Jaggi]{cordonnier2020on}
Jean-Baptiste Cordonnier, Andreas Loukas, and Martin Jaggi.
\newblock On the relationship between self-attention and convolutional layers.
\newblock In \emph{International Conference on Learning Representations}, 2020.
\newblock URL \url{https://openreview.net/forum?id=HJlnC1rKPB}.

\bibitem[Dai et~al.(2019)Dai, Yang, Yang, Cohen, Carbonell, Le, and
  Salakhutdinov]{dai2019transformer}
Zihang Dai, Zhilin Yang, Yiming Yang, William~W Cohen, Jaime Carbonell, Quoc~V
  Le, and Ruslan Salakhutdinov.
\newblock Transformer-xl: Attentive language models beyond a fixed-length
  context.
\newblock \emph{arXiv preprint arXiv:1901.02860}, 2019.

\bibitem[D'Ascoli et~al.(2021)D'Ascoli, Touvron, Leavitt, Morcos, Biroli, and
  Sagun]{dascoli2021convit}
St{\'e}phane D'Ascoli, Hugo Touvron, Matthew~L Leavitt, Ari~S Morcos, Giulio
  Biroli, and Levent Sagun.
\newblock Convit: Improving vision transformers with soft convolutional
  inductive biases.
\newblock In Marina Meila and Tong Zhang (eds.), \emph{Proceedings of the 38th
  International Conference on Machine Learning}, volume 139 of
  \emph{Proceedings of Machine Learning Research}, pp.\  2286--2296. PMLR,
  18--24 Jul 2021.

\bibitem[Devlin et~al.(2019)Devlin, Chang, Lee, and Toutanova]{devlin2019bert}
Jacob Devlin, Ming-Wei Chang, Kenton Lee, and Kristina Toutanova.
\newblock Bert: Pre-training of deep bidirectional transformers for language
  understanding.
\newblock In \emph{Proceedings of the 2019 Conference of the North American
  Chapter of the Association for Computational Linguistics: Human Language
  Technologies, Volume 1 (Long and Short Papers)}, pp.\  4171--4186, 2019.

\bibitem[Dong et~al.(2021)Dong, Cordonnier, and Loukas]{dong2021attention}
Yihe Dong, Jean-Baptiste Cordonnier, and Andreas Loukas.
\newblock Attention is not all you need: Pure attention loses rank doubly
  exponentially with depth.
\newblock \emph{arXiv preprint arXiv:2103.03404}, 2021.

\bibitem[Dosovitskiy et~al.(2021)Dosovitskiy, Beyer, Kolesnikov, Weissenborn,
  Zhai, Unterthiner, Dehghani, Minderer, Heigold, Gelly, Uszkoreit, and
  Houlsby]{dosovitskiy2021an}
Alexey Dosovitskiy, Lucas Beyer, Alexander Kolesnikov, Dirk Weissenborn,
  Xiaohua Zhai, Thomas Unterthiner, Mostafa Dehghani, Matthias Minderer, Georg
  Heigold, Sylvain Gelly, Jakob Uszkoreit, and Neil Houlsby.
\newblock An image is worth 16x16 words: Transformers for image recognition at
  scale.
\newblock In \emph{International Conference on Learning Representations}, 2021.
\newblock URL \url{https://openreview.net/forum?id=YicbFdNTTy}.

\bibitem[He et~al.(2016)He, Zhang, Ren, and Sun]{he2016deep}
Kaiming He, Xiangyu Zhang, Shaoqing Ren, and Jian Sun.
\newblock Deep residual learning for image recognition.
\newblock In \emph{Proceedings of the IEEE conference on computer vision and
  pattern recognition}, pp.\  770--778, 2016.

\bibitem[Huang et~al.(2020)Huang, Perez, Ba, and Volkovs]{huang2020improving}
Xiao~Shi Huang, Felipe Perez, Jimmy Ba, and Maksims Volkovs.
\newblock Improving transformer optimization through better initialization.
\newblock In \emph{International Conference on Machine Learning}, pp.\
  4475--4483. PMLR, 2020.

\bibitem[Ke et~al.(2020)Ke, He, and Liu]{ke2020rethinking}
Guolin Ke, Di~He, and Tie-Yan Liu.
\newblock Rethinking positional encoding in language pre-training.
\newblock In \emph{International Conference on Learning Representations}, 2020.

\bibitem[Krizhevsky et~al.(2009)]{krizhevsky2009learning}
Alex Krizhevsky et~al.
\newblock Learning multiple layers of features from tiny images.
\newblock 2009.

\bibitem[Lei~Ba et~al.(2016)Lei~Ba, Kiros, and Hinton]{lei2016layer}
Jimmy Lei~Ba, Jamie~Ryan Kiros, and Geoffrey~E Hinton.
\newblock Layer normalization.
\newblock \emph{arXiv preprint arXiv:1607.06450}, 2016.

\bibitem[Levine et~al.(2020)Levine, Wies, Sharir, Bata, and
  Shashua]{levine2020depth}
Yoav Levine, Noam Wies, Or~Sharir, Hofit Bata, and Amnon Shashua.
\newblock The depth-to-width interplay in self-attention.
\newblock \emph{arXiv preprint arXiv:2006.12467}, 2020.

\bibitem[Li et~al.(2021)Li, Si, Li, Hsieh, and Bengio]{li2021learnable}
Yang Li, Si~Si, Gang Li, Cho-Jui Hsieh, and Samy Bengio.
\newblock Learnable fourier features for multi-dimensional spatial positional
  encoding.
\newblock In \emph{NeurIPS}, 2021.

\bibitem[Likhosherstov et~al.(2021)Likhosherstov, Choromanski, and
  Weller]{likhosherstov2021expressive}
Valerii Likhosherstov, Krzysztof Choromanski, and Adrian Weller.
\newblock On the expressive power of self-attention matrices.
\newblock \emph{arXiv preprint arXiv:2106.03764}, 2021.

\bibitem[Liu et~al.(2019)Liu, Ott, Goyal, Du, Joshi, Chen, Levy, Lewis,
  Zettlemoyer, and Stoyanov]{liu2019roberta}
Yinhan Liu, Myle Ott, Naman Goyal, Jingfei Du, Mandar Joshi, Danqi Chen, Omer
  Levy, Mike Lewis, Luke Zettlemoyer, and Veselin Stoyanov.
\newblock Roberta: A robustly optimized bert pretraining approach.
\newblock \emph{arXiv preprint arXiv:1907.11692}, 2019.

\bibitem[Liu et~al.(2021)Liu, Lin, Cao, Hu, Wei, Zhang, Lin, and
  Guo]{liu2021swin}
Ze~Liu, Yutong Lin, Yue Cao, Han Hu, Yixuan Wei, Zheng Zhang, Stephen Lin, and
  Baining Guo.
\newblock Swin transformer: Hierarchical vision transformer using shifted
  windows.
\newblock \emph{arXiv preprint arXiv:2103.14030}, 2021.

\bibitem[Loshchilov \& Hutter(2018)Loshchilov and Hutter]{loshchilov2018fixing}
Ilya Loshchilov and Frank Hutter.
\newblock Fixing weight decay regularization in adam, 2018.

\bibitem[Luo et~al.(2021)Luo, Li, Cai, He, Peng, Zheng, Ke, Wang, and
  Liu]{luo2021stable}
Shengjie Luo, Shanda Li, Tianle Cai, Di~He, Dinglan Peng, Shuxin Zheng, Guolin
  Ke, Liwei Wang, and Tie-Yan Liu.
\newblock Stable, fast and accurate: Kernelized attention with relative
  positional encoding.
\newblock In \emph{NeurIPS}, 2021.

\bibitem[Paszke et~al.(2019)Paszke, Gross, Massa, Lerer, Bradbury, Chanan,
  Killeen, Lin, Gimelshein, Antiga, et~al.]{paszke2019pytorch}
Adam Paszke, Sam Gross, Francisco Massa, Adam Lerer, James Bradbury, Gregory
  Chanan, Trevor Killeen, Zeming Lin, Natalia Gimelshein, Luca Antiga, et~al.
\newblock Pytorch: An imperative style, high-performance deep learning library.
\newblock \emph{Advances in Neural Information Processing Systems},
  32:\penalty0 8026--8037, 2019.

\bibitem[Pérez et~al.(2019)Pérez, Marinković, and Barceló]{perez2018on}
Jorge Pérez, Javier Marinković, and Pablo Barceló.
\newblock On the turing completeness of modern neural network architectures.
\newblock In \emph{International Conference on Learning Representations}, 2019.
\newblock URL \url{https://openreview.net/forum?id=HyGBdo0qFm}.

\bibitem[Raffel et~al.(2020)Raffel, Shazeer, Roberts, Lee, Narang, Matena,
  Zhou, Li, and Liu]{raffel2020exploring}
Colin Raffel, Noam Shazeer, Adam Roberts, Katherine Lee, Sharan Narang, Michael
  Matena, Yanqi Zhou, Wei Li, and Peter~J Liu.
\newblock Exploring the limits of transfer learning with a unified text-to-text
  transformer.
\newblock \emph{Journal of Machine Learning Research}, 21:\penalty0 1--67,
  2020.

\bibitem[Touvron et~al.(2021)Touvron, Cord, Douze, Massa, Sablayrolles, and
  Jegou]{touvron2021deit}
Hugo Touvron, Matthieu Cord, Matthijs Douze, Francisco Massa, Alexandre
  Sablayrolles, and Herve Jegou.
\newblock Training data-efficient image transformers \& distillation through
  attention.
\newblock In Marina Meila and Tong Zhang (eds.), \emph{Proceedings of the 38th
  International Conference on Machine Learning}, volume 139 of
  \emph{Proceedings of Machine Learning Research}, pp.\  10347--10357. PMLR,
  18--24 Jul 2021.
\newblock URL \url{https://proceedings.mlr.press/v139/touvron21a.html}.

\bibitem[Vaswani et~al.(2017)Vaswani, Shazeer, Parmar, Uszkoreit, Jones, Gomez,
  Kaiser, and Polosukhin]{vaswani2017attention}
Ashish Vaswani, Noam Shazeer, Niki Parmar, Jakob Uszkoreit, Llion Jones,
  Aidan~N Gomez, \L~ukasz Kaiser, and Illia Polosukhin.
\newblock Attention is all you need.
\newblock In I.~Guyon, U.~V. Luxburg, S.~Bengio, H.~Wallach, R.~Fergus,
  S.~Vishwanathan, and R.~Garnett (eds.), \emph{Advances in Neural Information
  Processing Systems}, volume~30. Curran Associates, Inc., 2017.
\newblock URL
  \url{https://proceedings.neurips.cc/paper/2017/file/3f5ee243547dee91fbd053c1c4a845aa-Paper.pdf}.

\bibitem[Wies et~al.(2021)Wies, Levine, Jannai, and Shashua]{wies2021which}
Noam Wies, Yoav Levine, Daniel Jannai, and Amnon Shashua.
\newblock Which transformer architecture fits my data? a vocabulary bottleneck
  in self-attention.
\newblock \emph{arXiv preprint arXiv:2105.03928}, 2021.

\bibitem[Wightman(2019)]{rw2019timm}
Ross Wightman.
\newblock Pytorch image models.
\newblock \url{https://github.com/rwightman/pytorch-image-models}, 2019.

\bibitem[Xiong et~al.(2020)Xiong, Yang, He, Zheng, Zheng, Xing, Zhang, Lan,
  Wang, and Liu]{xiong2020layer}
Ruibin Xiong, Yunchang Yang, Di~He, Kai Zheng, Shuxin Zheng, Chen Xing,
  Huishuai Zhang, Yanyan Lan, Liwei Wang, and Tieyan Liu.
\newblock On layer normalization in the transformer architecture.
\newblock In \emph{International Conference on Machine Learning}, pp.\
  10524--10533. PMLR, 2020.

\bibitem[Yu \& Koltun(2015)Yu and Koltun]{yu2015multi}
Fisher Yu and Vladlen Koltun.
\newblock Multi-scale context aggregation by dilated convolutions.
\newblock \emph{arXiv preprint arXiv:1511.07122}, 2015.

\bibitem[Yun et~al.(2019)Yun, Bhojanapalli, Rawat, Reddi, and
  Kumar]{yun2019transformers}
Chulhee Yun, Srinadh Bhojanapalli, Ankit~Singh Rawat, Sashank Reddi, and Sanjiv
  Kumar.
\newblock Are transformers universal approximators of sequence-to-sequence
  functions?
\newblock In \emph{International Conference on Learning Representations}, 2019.

\bibitem[Yun et~al.(2020)Yun, Chang, Bhojanapalli, Rawat, Reddi, and
  Kumar]{yun2020on}
Chulhee Yun, Yin-Wen Chang, Srinadh Bhojanapalli, Ankit~Singh Rawat, Sashank~J.
  Reddi, and Sanjiv Kumar.
\newblock O$(n)$ connections are expressive enough: Universal approximability
  of sparse transformers.
\newblock In \emph{NeurIPS}, 2020.
\newblock URL
  \url{https://proceedings.neurips.cc/paper/2020/hash/9ed27554c893b5bad850a422c3538c15-Abstract.html}.

\bibitem[Zaheer et~al.(2020)Zaheer, Guruganesh, Dubey, Ainslie, Alberti,
  Ontanon, Pham, Ravula, Wang, Yang, et~al.]{zaheer2020big}
Manzil Zaheer, Guru Guruganesh, Kumar~Avinava Dubey, Joshua Ainslie, Chris
  Alberti, Santiago Ontanon, Philip Pham, Anirudh Ravula, Qifan Wang, Li~Yang,
  et~al.
\newblock Big bird: Transformers for longer sequences.
\newblock In \emph{NeurIPS}, 2020.

\end{thebibliography}
\bibliographystyle{iclr2022_conference}

\newpage

\appendix
\section*{Appendix}
\section{Omitted proofs of theoretical results}
\label{appendix:proofs}

\subsection{Proof of Lemma \ref{lemma:rpe_attn}}
\begin{proof}
    Recall that $\boldsymbol B_{i,j}=b_{(x_i-x_j, y_i-y_j)}$  where $(x_{\ell}, y_{\ell})$ denotes the position the $\ell$-th patch. Set $b_{\delta_0}=M$ and $b_{\delta}=0(\delta\neq \delta_0)$, where $M$ is a scalar. Then
    \begin{equation}
        \mathrm{softmax}(\boldsymbol B_{q,:})_k=\left\{
        \begin{array}{ll}
            \frac{1}{\mathrm{e}^M+N-1} &  q-k\neq \delta\\
            \frac{\mathrm{e}^M}{\mathrm{e}^M+N-1} &  q-k=\delta
        \end{array}
        \right.
    \end{equation}
   
    Note that
    \begin{align}
        \lim_{M\to +\infty} \frac{\mathrm{e}^M}{\mathrm{e}^M+N-1} &= 1.\\
        \lim_{M\to +\infty} \frac{1}{\mathrm{e}^M+N-1} &= 0.
    \end{align}
    Therefore, we only need to set $M$ to be sufficiently large number to conclude the proof. For example, by setting $M=40$ we will have $\mathrm{softmax}(\boldsymbol B_{q,:})_k=\mathbbm{1}_{\{q-k=\delta\}}$ up to machine precision.
\end{proof}

\subsection{Proof of Theorem \ref{thm:patch_theorem}}
\label{appendix:proof_the_thm}
\begin{proof}
Assume the input (sequence of flattened image patches) is $\boldsymbol X$.
We only need to prove the result for $d_H=d$ and $d_O=P^2D_{out}$, since an MHSA layer with larger $d_H$ and/or $d_O$ is at least as expressive as the one with $d_H=d$ and $d_O=P^2D_{out}$.

Define the receptive field of a given patch in $K\times K$ convolution be the set of patches which contain at least one pixel in the receptive field of any pixel in the given patch. Then it's easy to see that the relative position between a given patch and the patches in its receptive field are
\begin{equation}
    \tilde \Delta = \left\{-\left\lceil \frac{K-1}{2P} \right\rceil, \cdots, \left\lceil \frac{K-1}{2P} \right\rceil \right\} \times \left\{-\left\lceil \frac{K-1}{2P} \right\rceil, \cdots, \left\lceil \frac{K-1}{2P} \right\rceil\right\}.
\end{equation}

Note that $N_H=|\tilde \Delta|$. Therefore, for any relative position index $\delta\in \tilde \Delta$, we can assign an attention head for it, such that the query patch always attends to the patch at the given relative position $\delta$ in this head. We further set $\boldsymbol{W}^V_k=(\boldsymbol{I}_{d})$ (recall that $d_H=d$). Consequently, the hidden representation (before the final projection $\boldsymbol{W}^O$) of the query patch is the concatenation of the input features of all the patches in its receptive field. Precisely speaking, in Equation \ref{eqn:def_mhsa}, we have 
\begin{equation}
    \mathrm{concat}(\mathrm{SA}_1(\boldsymbol X),\cdots, \mathrm{SA}_{N_H}(\boldsymbol X))_{q,:}=\mathrm{concat}(\boldsymbol X_{q+\delta,:})_{\delta\in \tilde \Delta}.
\end{equation}



In the convolutional layer defined by $\boldsymbol W^C$, the output feature of any pixel in the $q$-th patch is a \textbf{linear function} of $\mathrm{concat}(\boldsymbol X_{q+\delta,:})_{\delta\in \tilde \Delta}$. So the output feature of the whole patch is also a linear function of $\mathrm{concat}(\boldsymbol X_{q+\delta,:})_{\delta\in \tilde \Delta}$. Therefore, there exists a linear projection matrix $\boldsymbol W^O$ such that 
\begin{equation}
    \mathrm{MHSA}(\boldsymbol X)_q = \mathrm{concat}(\boldsymbol X_{q+\delta,:})_{\delta\in \tilde \Delta}\boldsymbol W^O = \mathrm{conv}(\boldsymbol X)_q
\end{equation}

Moreover, due to the translation invariance property of the convolution operation, the linear projection matrix $\boldsymbol W^O$ \emph{does not depend on} $q$. Therefore, $\mathrm{MHSA}(\boldsymbol X)_q= \mathrm{conv}(\boldsymbol X)_q$ holds for any $q$. In other words, $\mathrm{MHSA}(\boldsymbol X)= \mathrm{conv}(\boldsymbol X)$.
\end{proof}

\paragraph{Remark.}
Indeed, we can presents $\boldsymbol W^O$ constructively: Assume $r\in[N_H];~s,t\in[P^2];~i\in[D_{in}];~j\in[D_{out}]$. Then $\boldsymbol W^O_{(r-1)d+(s-1)D_{in}+i, (t-1)D_{out}+j}=\boldsymbol W^C_{x(r,s,t),y(r,s,t),i,j}$, where $x(r,s,t), y(r,s,t)\in[K]\cup\{0\}$ are defined as follows:

Let $q$ be a patch on the image, and let $\tilde\Delta=\{\delta_1,\cdots, \delta_{N_H}\}$.
When the $s$-th pixel in the $(q+\delta_r)$-th patch is in the receptive field of the $t$-th pixel in the $q$-th patch, we use $(x(r,s,t),y(r,s,t))$ to denote its location in the receptive field. Otherwise, we let $(x(r,s,t),y(r,s,t))=(0,0)$, and define $\boldsymbol W^C_{x(r,s,t),y(r,s,t),i,j}=0$ in this case.

This construction will be useful in our experiment, which requires to transfer the knowledge of a convolutional layer into an MHSA layer (Section \ref{sec:exp_method}).

\subsection{Proof of Theorem \ref{thm:patch_lower_bound}}
\label{appendix:proof_patch_lower_bound}
\begin{proof}
    We will prove the theorem in the case where $D_{out}=1$, and consequently the result will hold for any $D_{out}\in \mathbb{N}^*$. Furthermore, we assume that $D_{in}=1$ since we can set $\boldsymbol{W}^C_{:,:,2:,:}\in \boldsymbol{0} \R^{K\times K\times (D_{in}-1)\times D_{out}}$ if $D_{in}>1$. In this way, the convolution computation will ignore all but the first channel.
    
    Assume the input (sequence of flattened image patches) is $\boldsymbol X$.
    Recall that a flattened patch is defined as the concatenation of the features of all the pixels in it, i.e., 
    \begin{equation}
        \boldsymbol{X}_{i,:} = \mathrm{concat}\left(\boldsymbol X_{h_{i1},w_{i1},:}, \cdots, \boldsymbol X_{h_{iP^2},w_{iP^2},:}\right)
    \end{equation}
    Since we have assumed that $D_{in}=1$, the feature of a pixel $\boldsymbol X_{h_{ip},w_{ip},:}$ is actually a scalar. Thus $\boldsymbol{\tilde X}_{i,:}\in\R^{P^{2}}$, i.e., $d=P^2$.
    
    If the output of the MHSA layer could express convolution, the output representation of a patch must contain the output representations of all the pixels in the convolutional layer. Again, since $D_{out}=1$, we can assume that the output dimension of the MHSA layer $d_O=P^2$. In other words, the output representation of a patch is the concatenation of the output representations of all its pixels.
   
    Therefore, $\boldsymbol{W}^V_k\boldsymbol{W}^O_k\in \R^{P^2\times P^2}~(\forall~k\in[N_H])$, and we let $\boldsymbol{W}^V_k\boldsymbol{W}^O_k=(w_{pq}^k)_{p,q\in[P^2]}$.
    
    In the MHSA layer, assume the attention distribution of query patch $\gamma$ in the $k$-th head is $\boldsymbol{a}^k(\gamma)=(a_{\delta}^k(\gamma))_{\gamma+\delta\in[H]\times[W]}$, where $\delta$ stands for the relative position between the query patch and the key patch.
    Consider the output feature of the pixel at position $q$ in patch $\gamma$ ($q$ denotes the location of the pixel on the patch, and $\gamma$ denotes the location of the patch on the image). We have
    \begin{align}
        \mathrm{MHSA}(\boldsymbol{X})_{{\gamma},q}=&\sum_{k=1}^{N_H}\sum_{\delta} a_{\delta}^k(\gamma) \sum_{q=1}^{P^2} \boldsymbol X_{\gamma+\delta,q} w_{pq}^k\\
        =&\sum_{\delta}\sum_{q=1}^{P^2} \boldsymbol X_{\gamma+\delta,q} \sum_{k=1}^{N_H} a_{\delta}^k(\gamma)  w_{pq}^k.
    \end{align}

    The above experssion is a linear transformation of $X$. In the convolutional layer, only pixels in the 9 neighboring patches (including the center patch itself) can be relevant, since $P>K$. Thus, $a_{\delta}^k(\gamma)>0$ only for $\delta\in \Delta=\{-1,0,1\}^2:=\{\delta_1, \cdots, \delta_9\}$.

    Let the (flattened) convolutional kernel $\boldsymbol{W}^C=(w^C_1, \cdots, w^C_{K^2})\in \R^{K^2}$, and additionally let $w^{C}_0=0$. Then for any $p,q\in \R^{P^2}, \delta\in\Delta$, we have
    \begin{equation}\label{eqn:patch_lower_bound_key}
        \sum_{h=1}^{N_H} a_{\delta}^h(\gamma)  w_{pq}^h = w^C_{k(p,q,\delta)},
    \end{equation}
    where $k(p,q,\delta)\in[K^2]\cup \{0\}$ is an index dependent on $p,q$ and $\delta$. $k(p,q,\delta)\neq 0$ if and only if the $q$-th pixel in the $\gamma+\delta$-th patch is in the receptive field of the $p$-th pixel in the $\gamma$-th patch. When $k(p,q,\delta)\neq 0$, the value of $k(p,q,\delta)$ only depends on the relative position between the two pixels.
    
    Let $\boldsymbol w_{pq}=(w_{pq}^1, \cdots, w_{pq}^{N_H})$, and 
    \begin{equation}
       \boldsymbol  W = \begin{pmatrix}
           \boldsymbol w_{11}\\
           \boldsymbol w_{12}\\
           \vdots \\
           \boldsymbol w_{PP}
       \end{pmatrix},
        A = \begin{pmatrix} 
            a_{\delta_1}^1& \cdots & a_{\delta_9}^1\\
            \vdots & &\vdots \\
            a_{\delta_1}^{N_H} & \cdots & a_{\delta_9}^{N_H} 
        \end{pmatrix},
        \tilde W^C = \begin{pmatrix} 
            w^C_{k(1,1,\delta_1)}& \cdots & w^C_{k(1,1,\delta_9)}\\
            \vdots & &\vdots \\
            w^C_{k(P,P,\delta_1)}& \cdots & w^C_{k(P,P,\delta_9)}\\
        \end{pmatrix}.
    \end{equation}
    Then Eqn \ref{eqn:patch_lower_bound_key} can be written in matrix form as $\boldsymbol {WA}=\boldsymbol {\tilde W}^C$.

    Since $P\geq K$, all the column in $\boldsymbol {\tilde W}^C$ is either a one-hot or a zero vector (pixels at the same position in the patches cannot be in the receptive field of \emph{one} pixel). Besides, none of the 9 rows is zero since they are all needed for the convolution computation. Therefore, we can select 9 columns in $\boldsymbol {\tilde W}^C$ and reorder them properly to form a diagonal sub-matrix of $\boldsymbol {\tilde W}^C$, which implies $rank(\boldsymbol {\tilde W}^C)=9$ as long as all the entries in the convolutional kernel is non-zero.

    On the other hand, $rank(\boldsymbol {WA})\leq rank(\boldsymbol {W})\leq N_H\leq 8$, which leads to a contradiction and concludes the proof.
\end{proof}

\section{Details on the experiment settings}
\label{appendix:exp}

In the experiments, we evaluate our CMHSA-3/5 models on CIFAR-100 \citep{krizhevsky2009learning}, using the proposed two-phase training pipeline. In both phases, the model is trained for 200 epochs with a 5-epoch warm-up stage followed by a cosine decay learning rate scheduler. In the convolutional phase, the patch projection layer is fixed as identity.

To train our models, we AdamW use as the optimizer, and set its hyperparameter $\varepsilon$ to $1e-8$ and $(\beta_1, \beta_2)$ to $(0,9, 0.999)$~\citep{loshchilov2018fixing}. 
We experiment with peak learning rate in $\{1e-4, 3e-4, 5e-4\}$ in the convolution phase, and $\{1e-5, 3e-5, 5e-5, 7e-5\}$ in the self-attention phase. 
The batch size is set to 128 in both phases.
We employ all the data augmentation and regularization strategies of \cite{touvron2021deit}, and remain all the relevant hyperparameters unmodified.

Our codes are implemented based on \texttt{PyTorch} \citep{paszke2019pytorch} and the \texttt{timm} library \citep{rw2019timm}. All the models are trained on 4 NVIDIA Tesla V100 GPUs with 16GB memory and the reported training time is also measured on these machines.

\end{document}